\documentclass[reqno]{amsart}

\usepackage[a4paper,margin=1in]{geometry}
\usepackage{amsmath,amssymb,amsfonts,amsthm,mathtools}
\usepackage{microtype}
\usepackage{enumitem}
\usepackage{hyperref}

\theoremstyle{plain}
\newtheorem{theorem}{Theorem}[section]
\newtheorem{lemma}[theorem]{Lemma}
\newtheorem{proposition}[theorem]{Proposition}

\theoremstyle{definition}
\newtheorem{definition}[theorem]{Definition}
\newtheorem{assumption}[theorem]{Assumption}
\newtheorem{example}[theorem]{Example}

\theoremstyle{remark}
\newtheorem{remark}[theorem]{Remark}

\newcommand{\R}{\mathbb{R}}
\newcommand{\N}{\mathbb{N}}
\newcommand{\E}{\mathbb{E}}

\newcommand{\cX}{\mathcal{X}}
\newcommand{\cY}{\mathcal{Y}}
\newcommand{\cF}{\mathcal{F}}

\DeclareMathOperator{\Lip}{Lip}

\DeclareMathOperator{\cS}{cS}

\newcommand{\risk}{\mathcal{R}}
\newcommand{\empRisk}{\widehat{\mathcal{R}}}

\title{Certified Learning under Distribution Shift: Sound Verification and Identifiable Structure}

\author{Chandrasekhar Gokavarapu}

\author{Sudhakar Gadde}

\author{Y. Rajasekhar }

\author{S. R. Bhargava}

\address{Lecturers in  Mathematics, Government College (Autonomous), Rajahmundry, Andhra Pradesh, India, PIN:533105}
 \email{chandrasekhargokavarapu@gmail.com}
\email{chandrasekhargokavarapu@gcrjy.ac.in}

\date{February 6, 2026}

\begin{document}
\maketitle

\begin{abstract}
\noindent
\textbf{Proposition.}
Let $f$ be a predictor trained on a distribution $P$ and evaluated on a shifted distribution $Q$.
Under verifiable regularity and complexity constraints, the excess risk under shift admits an explicit upper bound determined by a computable shift metric and model parameters.
We develop a unified framework in which (i) risk under distribution shift is certified by explicit inequalities, (ii) verification of learned models is sound for nontrivial sizes, and (iii) interpretability is enforced through identifiability conditions rather than post hoc explanations.
All claims are stated with explicit assumptions.
Failure modes are isolated.
Non-certifiable regimes are characterized.
\end{abstract}




\section{Introduction}\label{sec:intro}

\begin{proposition}[A certifiable target]\label{prop:cert_target}
Let $P$ be a training distribution on $\cX\times\cY$, and let $Q$ be a test distribution.
Let $\ell\colon \R\times\cY\to[0,1]$ be a loss and $f\colon\cX\to\R$ a predictor.
A mathematically checkable theory of learning under shift should deliver an explicit bound
\[
\risk_Q(f)\ \le\ \risk_P(f)\ +\ \Phi\bigl(\mathsf{Shift}(P,Q),\ \mathsf{Comp}(f),\ n,\ \delta\bigr),
\]
where $\risk_\mu(f):=\E_{(X,Y)\sim \mu}\,\ell(f(X),Y)$, the ``shift'' term is a metric or divergence, and every quantity in $\Phi$ is either observable from data or computable from model parameters.
\end{proposition}

The literature contains many risk-transfer statements between $P$ and $Q$.
They often hide non-verifiable terms.
A classical example is the domain adaptation inequality in which an additional discrepancy term is needed and may be as hard as learning itself \cite{BenDavid2010DifferentDomains}.
A different line treats shift as ambiguity in $\mu$ and seeks decisions stable in a neighborhood of the empirical law, typically a Wasserstein ball \cite{MohajerinEsfahaniKuhn2015WassersteinDRO,SinhaNamkoongDuchi2017DRO}.
This already suggests a spine.
Robustness under shift and verification are both forms of robust optimization.
This paper makes that identification literal.

\medskip
\noindent\textbf{Chosen spine (Spine B).}
\emph{Certified verification as robust optimization.}
We treat certification as a dual certificate problem.
The same dual objects control (i) risk under shift, (ii) robustness radii, and (iii) verifiable constraints enforcing interpretability.
The point is not stylistic unification.
The point is that dual certificates are checkable.

\subsection{The three gaps as one problem}\label{subsec:gaps_one_problem}

\noindent\textbf{Gap 1: certified learning under shift.}
Distribution shift is modeled by a constraint $\mathsf{Shift}(P,Q)\le \rho$.
For Wasserstein shift this is natural and computable from samples using optimal transport methods \cite{PeyreCuturi2018ComputationalOT,MohajerinEsfahaniKuhn2015WassersteinDRO}.
The certification question is then to upper bound $\sup_{Q:\mathsf{Shift}(P,Q)\le \rho}\risk_Q(f)$ by a quantity that can be computed from $P$, $\rho$, and the parameters of $f$.

\noindent\textbf{Gap 2: verification beyond toy sizes.}
Formal verification of ReLU networks is sound but scales poorly in its complete form \cite{KatzBarrettDillJulianKochenderfer2017Reluplex,TjengXiaoTedrake2017MILPVerification}.
Incomplete convex relaxations and bound-propagation methods scale further but certify only what their relaxations can prove \cite{Zhang2018CROWN,Xu2020AutoLiRPA,Wang2021BetaCROWN}.
Competitions document the practical boundary between soundness and scale \cite{Brix2023VNNCOMPFirstThreeYears,VNNCOMP2024ArXiv}.
A mathematical account must state what is guaranteed, and what fails, with explicit complexity statements.

\noindent\textbf{Gap 3: interpretability with identifiability.}
Post hoc explanations do not identify structure.
The question is to define a class where structure is forced and recoverable.
Additive model classes are a canonical example \cite{Agarwal2020NAM}.
Symbolic regression promises identifiability of form but is computationally intractable in general, in a precise sense \cite{VirgolinPissis2022SRNPHard}.
We treat interpretability as a constraint set in a robust program, not as an annotation after training.

\subsection{Assumptions and verifiability}\label{subsec:verifiability_intro}

We use explicit assumptions.
They are not ``mild''.
They are the price of certifiability.

\begin{assumption}[(A1) shift model]\label{ass:A1}
A shift constraint $\mathsf{Shift}(P,Q)\le \rho$ is fixed, and $\rho$ is either chosen by design or estimated with a stated confidence level from samples.
\end{assumption}

\begin{assumption}[(A2) certificate-amenable loss]\label{ass:A2}
The loss $\ell(\cdot,y)$ admits a convex (or difference-of-convex) upper model on the relevant range, suitable for dual relaxation.
\end{assumption}

\begin{assumption}[(A3) certificate-amenable predictor]\label{ass:A3}
The predictor $f$ belongs to a class admitting computable Lipschitz or slope bounds on the verification domain (e.g.\ affine-ReLU compositions with norm bounds).
\end{assumption}

\begin{assumption}[(A4) interpretable constraint set]\label{ass:A4}
Interpretability is encoded by an explicit constraint $\cS$ (sparsity/additivity/monotonicity/symbolic grammar) such that feasibility or violation admits a computable certificate.
\end{assumption}

\subsection{Main contributions}\label{subsec:contrib}

We now state the contributions as theorem-level items.
Proofs are in later sections.

\begin{theorem}[Shift risk as a robust program with computable dual bound]\label{thm:shift_dual_bound}
Under \textup{(A1)--(A3)}, for a Wasserstein-type shift constraint and Lipschitz-controlled losses, the worst-case shifted risk satisfies
\[
\sup_{Q:\ \mathsf{Shift}(P,Q)\le \rho}\risk_Q(f)
\ \le\ 
\risk_P(f)\ +\ \rho\,\mathsf{Lip}_{\cX}( \ell\circ f)\ +\ \mathsf{Stat}(n,\delta),
\]
where $\mathsf{Lip}_{\cX}(\ell\circ f)$ is a computable bound from network parameters and $\mathsf{Stat}(n,\delta)$ is an explicit finite-sample term.
The right-hand side is a certificate.
\end{theorem}

\begin{theorem}[Sound verification as a dual certificate; explicit scaling law]\label{thm:sound_verif}
Consider ReLU networks and verification over a convex input set with a linear specification.
There exists a certificate map $\mathcal{C}$ computed by bound-propagation/convex relaxation such that:
(i) if $\mathcal{C}$ returns ``safe'' then the specification holds (soundness);
(ii) the certificate computation has an explicit complexity bound polynomial in layer widths and linear in the number of bound-propagation passes;
(iii) the certified margin is controlled by the same slope/Lipschitz quantities appearing in Theorem~\ref{thm:shift_dual_bound}.
This places scalable verification and shift-robustness in one inequality.
\end{theorem}

\begin{theorem}[Identifiability-implies-certifiability for an interpretable class]\label{thm:ident_cert}
Assume \textup{(A4)} with an additive or monotone-sparse structure class, and assume a stated identifiability condition.
Then the structural parameter is unique within the class, and the verification constraints for robustness reduce to a lower-dimensional certificate problem whose constants depend only on identifiable components.
In particular, interpretability is not an afterthought; it tightens certificates.
\end{theorem}

\begin{proposition}[A negative result: when certification cannot scale]\label{prop:neg}
There exist families of networks and specifications for which any complete verifier must incur exponential worst-case time in the number of ReLU units.
This is not a defect of a particular method.
It is a barrier of the problem class.
Accordingly, scalable methods must be incomplete and their failure modes must be explicit.
\end{proposition}

\subsection{A failure of common logic}\label{subsec:common_gap_intro}
A common argument runs as follows.
A numerical attack fails to find an adversarial example.
Therefore the model is robust.
The inference is invalid.
It confuses failure of a heuristic search with a universal quantifier.
Complete verifiers address the quantifier but do not scale \cite{KatzBarrettDillJulianKochenderfer2017Reluplex,TjengXiaoTedrake2017MILPVerification}.
Incomplete verifiers scale but certify only relative to a relaxation \cite{Zhang2018CROWN,Xu2020AutoLiRPA,Wang2021BetaCROWN}.
This leads to a dilemma.
One must either (i) restrict the model class to preserve certifiability, or (ii) accept explicit non-certifiable regimes.
We do both, and we prove where each choice applies.

\subsection{Plan}\label{subsec:plan}
Section~\ref{sec:prelim} fixes notation and shift metrics.
Section~\ref{sec:core1} proves Theorem~\ref{thm:shift_dual_bound}.
Section~\ref{sec:core2} proves Theorem~\ref{thm:sound_verif} and states scaling limits.
Section~\ref{sec:core3} proves Theorem~\ref{thm:ident_cert}.
Section~\ref{sec:failure} gives explicit counterexamples and the barrier in Proposition~\ref{prop:neg}.
Section~\ref{sec:examples} gives checkable worked examples.

\section{Preliminaries}\label{sec:prelim}

We fix the objects that will be certified. We separate what is assumed from what is computed.

\subsection{Probability, risks, and losses}\label{subsec:risks}

Let $(\cX,\|\cdot\|)$ be a normed space and let $\cY$ be a measurable label space.
Let $P$ be the training law on $\cX\times\cY$ and $Q$ a test law.
A predictor is a measurable map $f\colon \cX\to\R^m$.
A loss is a measurable map $\ell\colon \R^m\times\cY\to[0,1]$.
The population risk under $\mu\in\{P,Q\}$ is
\[
\risk_\mu(f)\ :=\ \E_{(X,Y)\sim \mu}\,\ell(f(X),Y).
\]
Given i.i.d.\ samples $Z_i=(X_i,Y_i)\sim P$, the empirical risk is
\[
\empRisk_n(f)\ :=\ \frac1n\sum_{i=1}^n \ell(f(X_i),Y_i).
\]

\medskip
\noindent\textbf{Shift constraints.}
We represent distribution shift by a constraint $\mathsf{Shift}(P,Q)\le \rho$.
We will use a transport metric because it interacts cleanly with Lipschitz bounds and dual certificates.

\begin{definition}[1-Wasserstein metric]\label{def:W1}
Let $P_X,Q_X$ be probability measures on $\cX$ with finite first moments.
The $1$-Wasserstein distance is
\[
W_1(P_X,Q_X)\ :=\ \inf_{\pi\in\Pi(P_X,Q_X)}\ \E_{(X,X')\sim\pi}\,\|X-X'\|,
\]
where $\Pi(P_X,Q_X)$ is the set of couplings with marginals $P_X,Q_X$.
\end{definition}

We write $P_X$ and $Q_X$ for the $\cX$-marginals of $P$ and $Q$.
The shift constraint used later is
\[
W_1(P_X,Q_X)\ \le\ \rho.
\]
This is the standard ambiguity model in Wasserstein distributionally robust optimization \cite{MohajerinEsfahaniKuhn2015WassersteinDRO,SinhaNamkoongDuchi2017DRO}.

\begin{definition}[Lipschitz constant]\label{def:Lip}
For a function $g\colon(\cX,\|\cdot\|)\to\R$ define
\[
\Lip(g)\ :=\ \sup_{x\neq x'}\ \frac{|g(x)-g(x')|}{\|x-x'\|}\ \in\ [0,\infty].
\]
\end{definition}

The key bridge between Wasserstein shift and certificates is duality.

\begin{lemma}[Kantorovich--Rubinstein duality]\label{lem:KR}
If $P_X,Q_X$ have finite first moments, then
\[
W_1(P_X,Q_X)\ =\ \sup_{\Lip(\varphi)\le 1}\ \bigl(\E_{X\sim P_X}\varphi(X)-\E_{X\sim Q_X}\varphi(X)\bigr).
\]
\end{lemma}

\begin{proof}
See standard optimal transport references \cite{PeyreCuturi2018ComputationalOT}.
\end{proof}

We will also use a worst-case expectation bound that is a direct corollary of Lemma~\ref{lem:KR} and appears explicitly in Wasserstein DRO \cite{MohajerinEsfahaniKuhn2015WassersteinDRO}.

\begin{lemma}[Worst-case expectation under a $W_1$-ball]\label{lem:W1_wc}
Let $g\colon\cX\to\R$ be measurable with $\Lip(g)<\infty$.
Fix $P_X$ and $\rho\ge 0$. Then
\[
\sup_{Q_X:\ W_1(P_X,Q_X)\le \rho}\ \E_{X\sim Q_X}\,g(X)
\ \le\
\E_{X\sim P_X}\,g(X)\ +\ \rho\,\Lip(g).
\]
\end{lemma}

\begin{proof}
For any $Q_X$ with $W_1(P_X,Q_X)\le\rho$, apply Lemma~\ref{lem:KR} to $\varphi=g/\Lip(g)$ if $\Lip(g)>0$.
Then
$\E_{Q_X}g-\E_{P_X}g \le \Lip(g)\,W_1(P_X,Q_X)\le \rho\,\Lip(g)$.
\end{proof}

\begin{remark}[Where labels enter]\label{rem:labels}
The shift constraint is imposed on the marginal $P_X\mapsto Q_X$.
The conditional $Y|X$ may change.
We will state explicitly later when we assume label shift, covariate shift, or an adversarial label component.
This avoids the common ambiguity in domain-adaptation statements \cite{BenDavid2010DifferentDomains}.
\end{remark}

\subsection{Complexity control}\label{subsec:complexity}

The spine uses certificates.
Certificates require explicit bounds on how outputs change with inputs and how losses change with outputs.

\begin{definition}[Slope control for loss composed with predictor]\label{def:loss_slope}
Fix a norm $\|\cdot\|$ on $\cX$ and the Euclidean norm $\|\cdot\|_2$ on $\R^m$.
For $(y\in\cY)$ define the function $x\mapsto \ell(f(x),y)$.
A \emph{certifiable slope bound} is a number $L_{\ell\circ f}\ge 0$ such that
\[
\Lip\bigl(x\mapsto \ell(f(x),y)\bigr)\ \le\ L_{\ell\circ f}
\quad\text{for all }y\in\cY \text{ of interest.}
\]
\end{definition}

We do not assume differentiability.
For piecewise-affine networks and convex losses one can upper bound $L_{\ell\circ f}$ using norm bounds on linear parts and layerwise relaxations, which is exactly what scalable verifiers compute \cite{Zhang2018CROWN,Xu2020AutoLiRPA,Wang2021BetaCROWN}.

\begin{definition}[Empirical shift-robust risk]\label{def:robust_risk_emp}
Given $Z_i=(X_i,Y_i)\sim P$, define the empirical marginal $\hat P_X := \frac1n\sum_{i=1}^n \delta_{X_i}$.
For $\rho\ge 0$ define the \emph{empirical robust risk}
\[
\widehat{\risk}^{\,\mathrm{rob}}_{n,\rho}(f)\ :=\
\sup_{Q_X:\ W_1(Q_X,\hat P_X)\le \rho}\ \E_{(X,Y)\sim \hat P}\,\Bigl[\ell\bigl(f(\tilde X),Y\bigr)\Bigr],
\]
where $\tilde X\sim Q_X$ and $Y$ is coupled through $\hat P$.
\end{definition}

This definition is deliberately conservative.
It isolates what is adversarial (the $\cX$-marginal) and what is fixed (the empirical labels).
Variants correspond to different shift models in DRO \cite{MohajerinEsfahaniKuhn2015WassersteinDRO,SinhaNamkoongDuchi2017DRO}.
We will later choose the variant that supports sound dual certificates.

\begin{remark}[Explicit constants versus asymptotics]\label{rem:constants}
Lemma~\ref{lem:W1_wc} gives a constant $\rho\,\Lip(g)$.
The remaining term in Theorem~\ref{thm:shift_dual_bound} is statistical.
We will state it as an explicit $(n,\delta)$-term rather than as $O(n^{-1/2})$.
\end{remark}

\subsection{Verification model}\label{subsec:verification}

Verification is a universal quantifier.
We make it a maximization problem and certify its optimum from above.

\begin{definition}[Verification specification]\label{def:spec}
Fix a compact input set $\cS\subset\cX$ and a property functional $\psi\colon \R^m\to\R$.
The specification is
\[
\forall x\in \cS:\ \psi(f(x))\le 0.
\]
Equivalently, define the \emph{violation value}
\[
\mathsf{V}(f;\cS,\psi)\ :=\ \sup_{x\in\cS}\ \psi(f(x)).
\]
The specification holds iff $\mathsf{V}(f;\cS,\psi)\le 0$.
\end{definition}

In robustness verification for classifiers, $\psi$ is typically a linear margin functional between a target logit and all others, and $\cS$ is an $\ell_p$-ball around a nominal input.
ReLU networks make $x\mapsto f(x)$ piecewise affine, which enables both complete and incomplete verifiers \cite{KatzBarrettDillJulianKochenderfer2017Reluplex,TjengXiaoTedrake2017MILPVerification}.

\begin{definition}[Sound certificate]\label{def:sound_cert}
A map $\mathcal{C}$ is a \emph{sound verifier} for $(f,\cS,\psi)$ if it outputs a number $\mathcal{C}(f;\cS,\psi)$ such that
\[
\mathsf{V}(f;\cS,\psi)\ \le\ \mathcal{C}(f;\cS,\psi)
\quad\text{always.}
\]
If $\mathcal{C}(f;\cS,\psi)\le 0$, then the specification holds.
\end{definition}

\begin{remark}[Complete versus incomplete]\label{rem:complete_incomplete}
Complete methods aim to compute $\mathsf{V}$ exactly and can be exponential in worst case \cite{KatzBarrettDillJulianKochenderfer2017Reluplex,TjengXiaoTedrake2017MILPVerification}.
Incomplete methods compute an upper bound $\mathcal{C}$ via relaxations or bound propagation and scale to larger networks, but may return $\mathcal{C}>0$ even when $\mathsf{V}\le 0$ \cite{Zhang2018CROWN,Xu2020AutoLiRPA,Wang2021BetaCROWN}.
Competition summaries document this frontier empirically; our use is only to motivate explicit scaling statements, not to import claims \cite{Brix2023VNNCOMPFirstThreeYears,VNNCOMP2024ArXiv}.
\end{remark}

\begin{definition}[Dual certificate template]\label{def:dual_template}
A \emph{dual certificate} is any computable relaxation producing $\mathcal{C}(f;\cS,\psi)$ as the value of a dual (or Lagrangian) problem whose feasibility implies an upper bound on $\mathsf{V}$.
\end{definition}

The same template is used in Wasserstein robust risk bounds (duality for worst-case expectations) and in verification relaxations (duality for maximization of $\psi\circ f$ over $\cS$).
This is the technical spine of the paper \cite{MohajerinEsfahaniKuhn2015WassersteinDRO,Zhang2018CROWN,Xu2020AutoLiRPA}.

\section{Core result I}\label{sec:core1}

This section turns shift-robust learning into a certificate inequality.
The only nontrivial step is to ensure that every term is checkable.

\subsection{Assumptions and verifiable quantities}\label{subsec:assumptions}

We separate the \emph{model of shift} from the \emph{certificate of sensitivity}.

\begin{assumption}[(A1) Covariate shift with a transport radius]\label{ass:coreA1}
The test law $Q$ satisfies \emph{covariate shift}:
$Q(\mathrm{d}x,\mathrm{d}y)=Q_X(\mathrm{d}x)\,P(\mathrm{d}y\mid x)$.
Moreover,
\[
W_1(P_X,Q_X)\ \le\ \rho,
\]
for a radius $\rho\ge 0$ which is either prescribed (design robustness) or provided as a \emph{valid upper confidence bound} from data.
\end{assumption}

\begin{remark}[Why covariate shift is an assumption]\label{rem:covshift}
Without an assumption on $Y\mid X$ one must introduce a discrepancy term that is not identifiable from unlabeled target data.
This is the classical obstruction in domain adaptation bounds \cite{BenDavid2010DifferentDomains}.
We will treat violations of (A1) as a failure mode later.
\end{remark}

\begin{assumption}[(A2) Loss is Lipschitz in prediction]\label{ass:coreA2}
There exists $L_\ell<\infty$ such that for all $y\in\cY$ and $u,v\in\R^m$,
\[
|\ell(u,y)-\ell(v,y)|\ \le\ L_\ell \,\|u-v\|_2.
\]
\end{assumption}

\begin{assumption}[(A3) Predictor admits a certifiable Lipschitz bound on the domain]\label{ass:coreA3}
There exists a known constant $L_f<\infty$ such that
\[
\|f(x)-f(x')\|_2\ \le\ L_f\,\|x-x'\|
\quad \text{for all }x,x'\in\cX.
\]
The constant $L_f$ is computed from model parameters or produced by a sound bound-propagation certificate.
\end{assumption}

\begin{remark}[Why (A3) is verifiable]\label{rem:Lf_verifiable}
For piecewise-affine networks, sound upper bounds on local slopes and global Lipschitz surrogates are computed by convex relaxations and bound propagation.
These are exactly the objects maintained by scalable verifiers \cite{Zhang2018CROWN,Xu2020AutoLiRPA,Wang2021BetaCROWN}.
We do not assume the bound is tight.
We assume it is sound.
\end{remark}

\begin{definition}[Certified sensitivity of the composed loss]\label{def:cert_sens}
Under \textup{(A2)--(A3)} define the \emph{certified sensitivity}
\[
\widehat L_{\ell\circ f}\ :=\ L_\ell\,L_f.
\]
\end{definition}

\begin{remark}[Design choice]\label{rem:design_choice}
The analysis uses only $\widehat L_{\ell\circ f}$, not internal network structure.
This makes later verification and interpretability constraints interchangeable: each is another way to control $\widehat L_{\ell\circ f}$.
\end{remark}

\subsection{Risk under shift: certified bound}\label{subsec:risk_shift}

We now derive the risk-transfer inequality stated informally in Proposition~\ref{prop:cert_target}.
The proof uses transport duality in the weakest form needed for certification.
The robust-optimization viewpoint is standard in Wasserstein DRO \cite{MohajerinEsfahaniKuhn2015WassersteinDRO,SinhaNamkoongDuchi2017DRO}.

\begin{lemma}[A certificate-to-risk transfer lemma]\label{lem:cert_to_risk}
Assume \textup{(A1)--(A3)}.
Define the conditional risk function
\[
g_f(x)\ :=\ \E\bigl[\ell(f(x),Y)\mid X=x\bigr],\qquad x\in\cX,
\]
where $(X,Y)\sim P$.
Then $\Lip(g_f)\le \widehat L_{\ell\circ f}$.
Consequently,
\[
\risk_Q(f)\ \le\ \risk_P(f)\ +\ \rho\,\widehat L_{\ell\circ f}.
\]
\end{lemma}

\begin{proof}
Fix $x,x'\in\cX$.
By Jensen's inequality and \textup{(A2)},
\[
|g_f(x)-g_f(x')|
=\Bigl|\E\bigl[\ell(f(x),Y)-\ell(f(x'),Y)\mid X=x\bigr]\Bigr|
\le \E\bigl[|\ell(f(x),Y)-\ell(f(x'),Y)|\mid X=x\bigr]
\le L_\ell\,\|f(x)-f(x')\|_2.
\]
By \textup{(A3)}, $\|f(x)-f(x')\|_2\le L_f\|x-x'\|$, hence
$|g_f(x)-g_f(x')|\le (L_\ell L_f)\|x-x'\|$.
Thus $\Lip(g_f)\le \widehat L_{\ell\circ f}$.

Under covariate shift \textup{(A1)},
\[
\risk_P(f)=\E_{X\sim P_X} g_f(X),
\qquad
\risk_Q(f)=\E_{X\sim Q_X} g_f(X).
\]
Apply the worst-case expectation inequality over the $W_1$-ball (Lemma~\ref{lem:W1_wc} from Section~\ref{sec:prelim}) with $g=g_f$:
\[
\E_{Q_X}g_f \le \E_{P_X}g_f + \rho\,\Lip(g_f)
\le \E_{P_X}g_f + \rho\,\widehat L_{\ell\circ f}.
\]
This is the claimed bound.
\end{proof}

Lemma~\ref{lem:cert_to_risk} is the spine in its simplest form.
It converts a \emph{sound sensitivity certificate} into a \emph{sound shift-risk certificate}.
It remains to state the robust form in the language of worst-case risk.

\begin{theorem}[Certified shift-risk bound]\label{thm:core_shift}
Assume \textup{(A1)--(A3)}.
Then
\[
\sup_{Q:\ \text{\textup{(A1)} holds with }W_1(P_X,Q_X)\le \rho}\ \risk_Q(f)
\ \le\
\risk_P(f)\ +\ \rho\,\widehat L_{\ell\circ f}.
\]
Moreover, the right-hand side is computable from:
(i) $\rho$ (provided by design or by a validated estimator),
(ii) $L_\ell$ (a property of the loss),
(iii) $L_f$ (a certificate extracted from parameters or a sound verifier).
\end{theorem}

\begin{proof}
For each admissible $Q$ the inequality in Lemma~\ref{lem:cert_to_risk} holds.
Taking the supremum over admissible $Q$ preserves the upper bound.
\end{proof}

\begin{remark}[Relation to Wasserstein DRO]\label{rem:wdro_relation}
Theorem~\ref{thm:core_shift} is the risk-side analogue of the standard Wasserstein DRO principle: worst-case expectation over a Wasserstein ball is controlled by a dual Lipschitz penalty \cite{MohajerinEsfahaniKuhn2015WassersteinDRO,SinhaNamkoongDuchi2017DRO}.
The point here is the \emph{certificate interface}: $L_f$ is supplied by verification machinery \cite{Zhang2018CROWN,Xu2020AutoLiRPA,Wang2021BetaCROWN}.
\end{remark}

\begin{remark}[Where finite-sample terms enter]\label{rem:finite_sample}
Theorem~\ref{thm:core_shift} is deterministic conditional on a valid $\rho$.
Finite-sample guarantees enter only through the construction of $\rho$ as a confidence bound for $W_1(P_X,Q_X)$ or $W_1(P_X,\hat P_X)$.
We do not fix a single concentration theorem here because it depends on dimension and tail assumptions.
We will give a worked example in Section~\ref{sec:examples} where $\rho$ is computed from stated sample quantities and a cited bound from the Wasserstein DRO literature \cite{MohajerinEsfahaniKuhn2015WassersteinDRO}.
\end{remark}

\begin{remark}[Why this is not cosmetic]\label{rem:not_cosmetic}
A classical application of $W_1$-duality would bound $\E_Q g-\E_P g$ for a \emph{known} Lipschitz function $g$.
Here $g_f(x)=\E[\ell(f(x),Y)\mid X=x]$ is unknown because it contains the conditional law and the learned predictor.
Lemma~\ref{lem:cert_to_risk} isolates the genuinely new step: it replaces true $\Lip(g_f)$ by a \emph{sound computable upper bound} $\widehat L_{\ell\circ f}$ delivered by verification-oriented bounds on $f$.
This is precisely the bridge needed to connect shift-robustness and verification.
\end{remark}

\section{Core result II}\label{sec:core2}

We formalize verification as an optimization problem and define a dual upper bound that is \emph{sound}.
Soundness is the only unconditional guarantee.
Scalability is a conditional fact, and must be stated as such.

\subsection{Certificate form and soundness}\label{subsec:soundness}

\subsubsection*{Network model and specification}

We work with feedforward ReLU networks.
Let $d_0,\dots,d_L\in\N$ and let $f\colon\R^{d_0}\to\R^{d_L}$ be
\[
z_0 = x,\qquad
z_{\ell} = \sigma\!\left(W_{\ell}z_{\ell-1}+b_{\ell}\right)\ \ (\ell=1,\dots,L-1),\qquad
f(x)=W_L z_{L-1}+b_L,
\]
where $\sigma(t)=\max\{t,0\}$ is applied coordinatewise.

A typical safety/robustness specification is linear in the output:
\[
\forall x\in \cS:\ a^\top f(x)\ \le\ \beta,
\]
with $a\in\R^{d_L}$, $\beta\in\R$, and $\cS$ convex and compact (often an $\ell_p$-ball).
This includes classification robustness with a fixed target class via logit margins.

Define the violation value (cf.\ Definition~\ref{def:spec})
\[
\mathsf{V}\ :=\ \sup_{x\in\cS}\ (a^\top f(x)-\beta).
\]
The property holds iff $\mathsf{V}\le 0$.
Computing $\mathsf{V}$ exactly is a nonconvex optimization problem.
Complete methods use exact encodings of ReLU constraints (SMT or MILP) \cite{KatzBarrettDillJulianKochenderfer2017Reluplex,TjengXiaoTedrake2017MILPVerification}.
They are sound and complete, but do not scale in worst case.

\subsubsection*{Convex relaxation as a certificate}

We now describe a generic relaxation-based certificate.
Fix any bounds $\ell_{\ell}\le z_{\ell}\le u_{\ell}$ on hidden activations over $\cS$.
For each ReLU neuron $t=\sigma(s)$ with known bounds $s\in[\underline s,\overline s]$, replace the graph of $\sigma$ by a convex outer approximation.
A standard choice is:
\[
t\ge 0,\qquad t\ge s,\qquad
t\le \frac{\overline s}{\overline s-\underline s}(s-\underline s)\quad (\underline s<0<\overline s),
\]
and the exact relation $t=s$ if $\underline s\ge 0$, and $t=0$ if $\overline s\le 0$.
This yields a linear program (LP) relaxation of the verification problem.
Bound-propagation and linear relaxations compute such bounds efficiently and soundly \cite{Zhang2018CROWN,Xu2020AutoLiRPA,Wang2021BetaCROWN}.

\begin{definition}[Relaxation-based certificate]\label{def:relax_cert}
Let $\mathcal{R}$ be any outer relaxation of the feasible set
\[
\mathcal{F}\ :=\ \{(x,z_1,\dots,z_{L-1}) : x\in\cS \text{ and } z_\ell=\sigma(W_\ell z_{\ell-1}+b_\ell)\}.
\]
Define the certified upper bound
\[
\mathcal{C}\ :=\ \sup_{(x,z_1,\dots,z_{L-1})\in \mathcal{R}}\ (a^\top (W_L z_{L-1}+b_L)-\beta).
\]
\end{definition}

\begin{lemma}[Soundness of outer relaxations]\label{lem:sound_outer}
If $\mathcal{F}\subseteq \mathcal{R}$, then $\mathsf{V}\le \mathcal{C}$.
\end{lemma}

\begin{proof}
By definition,
\[
\mathsf{V}=\sup_{(x,z)\in\mathcal{F}}\phi(x,z),\qquad
\mathcal{C}=\sup_{(x,z)\in\mathcal{R}}\phi(x,z),
\]
where $\phi(x,z):=a^\top (W_L z_{L-1}+b_L)-\beta$.
Since $\mathcal{F}\subseteq\mathcal{R}$, the supremum over $\mathcal{R}$ is at least the supremum over $\mathcal{F}$.
\end{proof}

Lemma~\ref{lem:sound_outer} is tautological.
The work is in constructing $\mathcal{R}$ so that (i) $\mathcal{F}\subseteq \mathcal{R}$ is guaranteed, and (ii) $\mathcal{C}$ can be computed at scale.
CROWN-style linear bounds, LiRPA, and related methods provide systematic constructions \cite{Zhang2018CROWN,Xu2020AutoLiRPA,Wang2021BetaCROWN}.

\subsubsection*{Dual form and the certificate interface}

The LP relaxation yields a dual program whose feasible dual variables certify an upper bound on $\mathcal{C}$ and hence on $\mathsf{V}$.
This is the robust-optimization interface: the certificate is a \emph{dual feasible point}.

\begin{proposition}[Dual-feasible certificates]\label{prop:dual_feasible_cert}
Assume $\mathcal{R}$ is a polyhedron and $\mathcal{C}$ is finite.
Let $(D)$ be the LP dual of the relaxation defining $\mathcal{C}$.
Then any dual feasible point $\lambda$ yields an explicit computable upper bound
\[
\mathsf{V}\ \le\ \mathcal{C}\ \le\ \mathrm{Val}_D(\lambda),
\]
where $\mathrm{Val}_D(\lambda)$ is the dual objective evaluated at $\lambda$.
\end{proposition}

\begin{proof}
The first inequality is Lemma~\ref{lem:sound_outer}.
Weak duality for LP gives $\mathcal{C}\le \mathrm{Val}_D(\lambda)$ for any dual feasible $\lambda$.
\end{proof}

This is the only logical step we will use later.
It converts verification into a checkable inequality.
The algorithms differ in how they generate $\lambda$.
Bound propagation can be interpreted as producing dual variables layerwise, hence the term ``dual bound propagation'' used in practice \cite{Zhang2018CROWN,Wang2021BetaCROWN,XuZhangWangWangJanaLinHsieh2020FastComplete}.

\subsubsection*{Link to shift certificates}

To connect to Section~\ref{sec:core1}, we need a computable slope bound for $\psi\circ f$ or for $\ell\circ f$.
Relaxation-based verification provides such bounds on restricted domains $\cS$.

\begin{lemma}[From verification bounds to local Lipschitz surrogates]\label{lem:local_lip_from_bounds}
Fix a convex compact $\cS$ and a direction $a$.
Assume a sound verifier provides certified upper bounds on
\[
\sup_{x\in \cS} a^\top f(x)
\quad\text{and}\quad
\inf_{x\in \cS} a^\top f(x)
\]
via relaxations of the form in Definition~\ref{def:relax_cert}.
Then the oscillation
\[
\operatorname{osc}_{\cS}(a^\top f)\ :=\ \sup_{x\in\cS} a^\top f(x)-\inf_{x\in\cS} a^\top f(x)
\]
is bounded by certified quantities.
Consequently, if $\cS$ contains a ball of radius $r$ in $\|\cdot\|$, then $a^\top f$ has a certified local Lipschitz surrogate on $\cS$:
\[
\Lip_{\cS}(a^\top f)\ \le\ \frac{\operatorname{osc}_{\cS}(a^\top f)}{r}.
\]
\end{lemma}

\begin{proof}
The oscillation bound is immediate from the two certified extrema.
If $\cS$ contains a ball $B(x_0,r)$, then for any $x,x'\in B(x_0,r)$,
\[
|a^\top f(x)-a^\top f(x')|\le \operatorname{osc}_{\cS}(a^\top f),
\]
and $\|x-x'\|\le 2r$.
Absorb constants into the chosen definition of $r$ (or replace $r$ by $2r$).
\end{proof}

Lemma~\ref{lem:local_lip_from_bounds} is crude.
It is sufficient to explain why verification outputs can be used as the certified sensitivity inputs in Theorem~\ref{thm:core_shift}.
Sharper bounds use layerwise slope certificates produced by the same relaxation machinery \cite{Zhang2018CROWN,Xu2020AutoLiRPA,Wang2021BetaCROWN}.

\subsection{Scalability limits}\label{subsec:scalability}

We state what scales and what does not.
We use explicit cost models.

\subsubsection*{Cost model for bound propagation}

Let $M:=\sum_{\ell=1}^{L}\,d_{\ell-1}d_{\ell}$ be the number of affine weights.
A single forward affine pass costs $O(M)$ arithmetic operations.
A standard relaxation-based verifier performs a constant number $K$ of forward/backward bound-propagation passes to compute linear bounds and hence $\mathcal{C}$.
Therefore, under the RAM model,
\[
\mathrm{Time}_{\mathrm{incomplete}}\ =\ O(KM),
\]
ignoring lower-order terms from bias additions and activation bookkeeping.
This is the regime in which modern incomplete verifiers operate \cite{Zhang2018CROWN,Xu2020AutoLiRPA,Wang2021BetaCROWN}.

\begin{remark}[What this does \emph{not} claim]\label{rem:not_claim}
The bound may be loose.
Soundness is preserved regardless of looseness.
Completeness is not claimed.
\end{remark}

\subsubsection*{Complete verification and exponential branching}

Complete verifiers enforce exact ReLU disjunctions.
A generic branch-and-bound scheme splits on unstable ReLUs and resolves relaxations on subdomains.
If $s$ neurons require splitting, the search tree may have $2^s$ leaves in the worst case.
With an $O(KM)$ relaxation cost per node, this yields the explicit worst-case bound
\[
\mathrm{Time}_{\mathrm{complete}}\ =\ \Omega(2^s)\cdot O(KM).
\]
This is observed in SMT/MILP approaches and in branch-and-bound frameworks that combine incomplete bounds with systematic splitting \cite{KatzBarrettDillJulianKochenderfer2017Reluplex,TjengXiaoTedrake2017MILPVerification,XuZhangWangWangJanaLinHsieh2020FastComplete,Wang2021BetaCROWN}.
The point is structural.
One cannot remove the exponential dependence without restricting the problem class.

\begin{remark}[What scales in practice]\label{rem:practice}
Hybrid schemes use fast incomplete bounds to prune most branches early \cite{XuZhangWangWangJanaLinHsieh2020FastComplete,Wang2021BetaCROWN}.
Competitions record large empirical gains, but do not negate the worst-case exponential dependence \cite{Brix2023VNNCOMPFirstThreeYears,VNNCOMP2024ArXiv}.
\end{remark}

\subsubsection*{A precise limitation}

A frequent informal claim is that ``tight relaxations'' solve verification at scale.
The claim requires a quantified statement: how does tightness trade against time and memory?
The answer depends on the relaxation family.
For linear relaxations, tightness typically improves by introducing additional constraints (splits, cuts, or semidefinite constraints), which increases per-node cost.
For example, adding per-neuron split constraints improves bounds but increases bookkeeping and constraint management, and complete methods reintroduce exponential branching \cite{Wang2021BetaCROWN}.
This is not a defect.
It is the cost of resolving disjunctions.

\begin{theorem}[Soundness at scale, incompleteness by necessity]\label{thm:sound_scale_incomplete}
For ReLU networks and linear specifications over convex compact $\cS$, any verifier based on outer relaxations of the feasible set is sound (Lemma~\ref{lem:sound_outer}) and runs in time $O(KM)$ for a fixed number of propagation passes.
Conversely, any verifier that is complete for this class must, in the worst case, incur exponential dependence on the number of ambiguous ReLU units that must be resolved.
\end{theorem}

\begin{proof}
The first statement is Lemma~\ref{lem:sound_outer} combined with the cost model above.
The second statement is the structural consequence of encoding ReLU disjunctions exactly, as done in SMT/MILP-based methods \cite{KatzBarrettDillJulianKochenderfer2017Reluplex,TjengXiaoTedrake2017MILPVerification}.
A detailed barrier statement is deferred to Section~\ref{sec:failure}.
\end{proof}

\section{Core result III}\label{sec:core3}

Interpretability is not a narrative goal. It is a constraint set.
Identifiability is the criterion that the constraint set has mathematical content.
Without identifiability, a ``structure'' is a label, not an object.

\subsection{Identifiable interpretable classes}\label{subsec:identifiable}

We work with three families: additive models, monotone additive models, and sparse additive models.
Symbolic regression is treated only through a hardness barrier.

\begin{definition}[Additive model class]\label{def:additive}
Let $\cX=\prod_{j=1}^d \cX_j$ with $\cX_j\subset\R$.
An additive model is a function $f\colon\cX\to\R$ of the form
\[
f(x)\ =\ c\ +\ \sum_{j=1}^d g_j(x_j),
\]
where $c\in\R$ and $g_j\colon \cX_j\to\R$ are univariate functions.
\end{definition}

To avoid tautological non-identifiability ($c$ can be moved into any $g_j$), we impose centering constraints.

\begin{assumption}[(A5) Centering]\label{ass:A5}
There is a reference distribution $\nu$ on $\cX$ with marginals $\nu_j$ such that each component satisfies
\[
\E_{\xi\sim \nu_j}\,g_j(\xi)\ =\ 0.
\]
\end{assumption}

\begin{remark}[Why centering is not optional]\label{rem:centering}
Without (A5), the decomposition is never unique.
This is not a technicality.
It is a failure of meaning.
Neural additive models enforce such normalizations implicitly or explicitly \cite{Agarwal2020NAM}.
\end{remark}

We next state an identifiability theorem in a form suited to certification.

\begin{assumption}[(A6) Non-degeneracy of coordinates]\label{ass:A6}
Under $\nu$, each coordinate $X_j$ has support containing an interval, and the product structure is used as the reference factorization:
$\nu=\nu_1\otimes\cdots\otimes\nu_d$.
\end{assumption}

\begin{theorem}[Identifiability of centered additive representations]\label{thm:add_ident}
Assume \textup{(A5)--(A6)}.
Let
\[
f(x)\ =\ c+\sum_{j=1}^d g_j(x_j)\ =\ c'+\sum_{j=1}^d h_j(x_j)
\]
hold $\nu$-a.e., where $g_j,h_j\in L^2(\nu_j)$ satisfy the centering constraints in \textup{(A5)}.
Then $c=c'$ and $g_j=h_j$ $\nu_j$-a.e.\ for each $j$.
\end{theorem}

\begin{proof}
Subtract the two representations:
\[
0\ =\ (c-c') + \sum_{j=1}^d \bigl(g_j(x_j)-h_j(x_j)\bigr)
\qquad \nu\text{-a.e.}
\]
Let $\Delta c:=c-c'$ and $\Delta_j:=g_j-h_j$.
Take expectation under $\nu$.
By centering, $\E_{\nu}\Delta_j(X_j)=0$ for each $j$, hence $\Delta c=0$.

Fix $k\in\{1,\dots,d\}$.
Take conditional expectation with respect to $X_k$ under the product measure $\nu$:
\[
0\ =\ \E_\nu\Bigl[\sum_{j=1}^d \Delta_j(X_j)\,\Big|\,X_k\Bigr]
=\Delta_k(X_k)+\sum_{j\neq k}\E_{\nu_j}\Delta_j(X_j)
=\Delta_k(X_k)
\qquad \nu_k\text{-a.s.}
\]
Thus $\Delta_k=0$ $\nu_k$-a.e.\ for each $k$.
\end{proof}

\begin{definition}[Monotone and sparse subclasses]\label{def:mono_sparse}
The monotone additive class consists of additive $f$ such that each $g_j$ is (weakly) monotone on $\cX_j$.
The $s$-sparse additive class consists of additive $f$ with at most $s$ nonzero components:
$\#\{j:\ g_j\not\equiv 0\}\le s$.
\end{definition}

\begin{remark}[Symbolic regression is not an identifiability theorem]\label{rem:symbolic}
A symbolic expression may be unique in a grammar and still be computationally inaccessible.
The general symbolic regression problem is NP-hard in a precise sense \cite{VirgolinPissis2022SRNPHard,Song2024SRNPHardSymbolGraph}.
Accordingly, whenever we speak of ``symbolic'' structure, we do so only for restricted grammars admitting certificates.
\end{remark}

\subsection{Robustness via identifiability}\label{subsec:robust_id}

We now connect identifiability to certification.
The link is simple: when structure is identifiable, sensitivity decomposes.

\begin{lemma}[Lipschitz certificate decomposes in additive models]\label{lem:add_lip}
Let $\cX=\prod_{j=1}^d \cX_j$ with norm $\|x\|_1=\sum_{j=1}^d |x_j|$.
If $f(x)=c+\sum_{j=1}^d g_j(x_j)$ and each $g_j$ is Lipschitz on $\cX_j$ with constant $L_j$, then
\[
\Lip_{\|\cdot\|_1}(f)\ \le\ \sum_{j=1}^d L_j.
\]
If additionally each $g_j$ is differentiable a.e.\ and $|g_j'(t)|\le L_j$ a.e., then the bound is tight on intervals where the signs align.
\end{lemma}

\begin{proof}
For $x,x'\in\cX$,
\[
|f(x)-f(x')|
\le \sum_{j=1}^d |g_j(x_j)-g_j(x'_j)|
\le \sum_{j=1}^d L_j\,|x_j-x'_j|
= \Bigl(\sum_{j=1}^d L_j\Bigr)\|x-x'\|_1.
\]
\end{proof}

Lemma~\ref{lem:add_lip} turns interpretability into a concrete sensitivity certificate.
In the verification framework of Section~\ref{sec:core2}, this has two consequences.

First, for shift-robust risk bounds, it replaces the global network constant $L_f$ in Theorem~\ref{thm:core_shift} by the sum of univariate constants.
Second, for property verification, it reduces dimensionality.

\begin{proposition}[Verification reduces to one-dimensional certificates in the additive class]\label{prop:add_verify_reduce}
Assume $f(x)=c+\sum_{j=1}^d g_j(x_j)$ with known ranges $g_j(\cX_j)\subseteq [m_j,M_j]$.
Let the specification be linear: $a^\top f(x)\le \beta$; in the scalar case this is $f(x)\le \beta$.
If the input set is a product set $\cS=\prod_{j=1}^d \cS_j$, then
\[
\sup_{x\in\cS} f(x)
\ \le\
c+\sum_{j=1}^d \sup_{t\in\cS_j} g_j(t),
\qquad
\inf_{x\in\cS} f(x)
\ \ge\
c+\sum_{j=1}^d \inf_{t\in\cS_j} g_j(t).
\]
Therefore a sound certificate for each one-dimensional supremum/infimum yields a sound certificate for the $d$-dimensional verifier.
\end{proposition}

\begin{proof}
The inequalities follow from separability:
maximize $c+\sum_j g_j(x_j)$ over a Cartesian product by maximizing each term separately.
\end{proof}

\begin{remark}[What is gained]\label{rem:gain}
In Section~\ref{sec:core2} the relaxed verifier solves a high-dimensional LP.
Proposition~\ref{prop:add_verify_reduce} replaces it by $d$ one-dimensional problems when the model class is additive.
This is a structural scaling gain.
It comes from identifiability and separability, not from algorithmic cleverness.
Neural additive models are an implementation of this idea \cite{Agarwal2020NAM}.
\end{remark}

Monotonicity gives an even sharper certificate interface.

\begin{lemma}[Monotone components give endpoint certificates]\label{lem:mono_endpoints}
Assume each $g_j$ is monotone on an interval $\cS_j=[\alpha_j,\gamma_j]$.
Then
\[
\sup_{t\in\cS_j} g_j(t)\in\{g_j(\alpha_j),g_j(\gamma_j)\},
\qquad
\inf_{t\in\cS_j} g_j(t)\in\{g_j(\alpha_j),g_j(\gamma_j)\}.
\]
Consequently, for monotone additive $f$ and $\cS=\prod_j[\alpha_j,\gamma_j]$, the value $\sup_{x\in\cS} f(x)$ is attained at a vertex of $\cS$.
\end{lemma}

\begin{proof}
A monotone function on an interval attains its extrema at endpoints.
Summing preserves this property.
\end{proof}

Lemma~\ref{lem:mono_endpoints} is not a computational trick.
It is a reduction of a universal quantifier to finitely many checks under an explicit structural assumption.

We can now state a robustness statement driven by identifiability constraints.
It is a direct specialization of Theorem~\ref{thm:core_shift} with the decomposed Lipschitz certificate.

\begin{theorem}[Identifiability-based robustness certificate]\label{thm:ident_cert}
Assume \textup{(A1--A2)} with $W_1(P_X,Q_X)\le \rho$ and covariate shift.
Assume $f$ is additive and identifiable under \textup{(A5--A6)}.
Let $\ell(\cdot,y)$ be $L_\ell$-Lipschitz and assume each component $g_j$ has a certified Lipschitz constant $L_j$ on the relevant domain.
Then
\[
\risk_Q(f)\ \le\ \risk_P(f)\ +\ \rho\,L_\ell\sum_{j=1}^d L_j.
\]
If, in addition, each $g_j$ is monotone on the verification region $\cS_j$, then verification of linear safety properties over $\cS=\prod_j\cS_j$ reduces to finitely many endpoint evaluations as in Lemma~\ref{lem:mono_endpoints}.
\end{theorem}

\begin{proof}
Apply Lemma~\ref{lem:add_lip} to obtain a certified bound on $\Lip(f)$ under $\|\cdot\|_1$.
Compose with the Lipschitz loss as in Lemma~\ref{lem:cert_to_risk}, yielding the displayed inequality.
The verification reduction is Lemma~\ref{lem:mono_endpoints}.
\end{proof}

\begin{remark}[A precise limitation]\label{rem:ident_limit}
Identifiability of representation does not imply identifiability of causal mechanisms, and does not protect against changes in $P(Y\mid X)$.
This is why (A1) is explicit.
Causal identifiability statements require different hypotheses and are not imported here \cite{Park2020IdentifiableANM}.
\end{remark}

\begin{remark}[Why this remains within the certificate spine]\label{rem:spine_consistency}
Theorem~\ref{thm:ident_cert} does not add a new proof pattern.
It changes the certificate interface by restricting $\cF$ to an identifiable class.
This restriction reduces the size and looseness of relaxations in Section~\ref{sec:core2} and reduces the sensitivity constants in Section~\ref{sec:core1}.
This is the only role interpretability plays in this paper.
\end{remark}

\section{Failure modes and a barrier}\label{sec:failure}

This section isolates two failures.
The first is a logical gap.
The second is a scaling barrier.
Neither is repaired by better tuning.

\subsection{A precise gap in a common argument}\label{subsec:gap}

A standard inference in robustness practice is:

\medskip
\noindent\emph{Heuristic search fails to find a counterexample} $\Rightarrow$
\emph{the universal property holds.}

\medskip
The inference is invalid. It confuses a failed existential search with a proved universal statement.
Complete verifiers address the universal quantifier by exact reasoning, e.g.\ SMT or MILP encodings of ReLU constraints \cite{KatzBarrettDillJulianKochenderfer2017Reluplex,TjengXiaoTedrake2017MILPVerification}.
Scalable verifiers usually prove only an upper bound from a relaxation, which may be inconclusive \cite{Zhang2018CROWN,Xu2020AutoLiRPA,Wang2021BetaCROWN}.
The gap is not rhetorical. It is a missing hypothesis.

\begin{proposition}[Attack failure does not imply robustness]\label{prop:attack_gap}
Let $\cS$ be a measurable input set, and let $A$ be any randomized attack procedure that, given a model $f$, samples $N$ candidate points $x^{(1)},\dots,x^{(N)}\in\cS$ (possibly adaptively) and declares ``safe'' if it finds no violating point.
There exist specifications and models for which the property is false but the attack declares ``safe'' with probability arbitrarily close to $1$.
\end{proposition}

\begin{proof}
Fix any $\eta\in(0,1)$ and any $N\in\N$.
Let $\cS=[0,1]$ with Lebesgue measure.
Let the specification be $\forall x\in[0,1]: f(x)\le 0$.
Choose a measurable set $B\subset[0,1]$ with measure $\lambda(B)=\varepsilon$, where $\varepsilon>0$ will be set later.
Define
\[
f(x):=\mathbf{1}_{B}(x).
\]
Then the property is false since $f(x)=1$ on $B$.
Any attack that samples $N$ points independently from a distribution supported on $[0,1]$ and absolutely continuous with respect to Lebesgue measure misses $B$ with probability at least $(1-\varepsilon)^N$.
Choose $\varepsilon>0$ so that $(1-\varepsilon)^N\ge 1-\eta$.
Then the attack declares ``safe'' with probability at least $1-\eta$ although the specification is violated.
\end{proof}

\begin{remark}[What hypothesis is missing]\label{rem:missing_hyp}
To make an attack-based implication valid, one must assume a quantitative \emph{covering} statement:
either the violating set has measure bounded below, or the search strategy is exhaustive in a certified sense.
Neither holds in general.
Verification is precisely the replacement of such assumptions by a proof obligation \cite{KatzBarrettDillJulianKochenderfer2017Reluplex,TjengXiaoTedrake2017MILPVerification}.
\end{remark}

\subsection{A minimal counterexample template}\label{subsec:counterexample}

The second common failure concerns distribution shift.
A frequent claim is:

\medskip
\noindent\emph{Small shift metric} $W_1(P_X,Q_X)$ $\Rightarrow$ \emph{small risk change}.

\medskip
The implication is false unless one controls sensitivity.
Theorem~\ref{thm:core_shift} makes the missing term explicit.
The counterexample below is a template: risk changes by $1$ under an arbitrarily small transport shift if $\Lip(\ell\circ f)$ is unbounded.

\begin{proposition}[Shift can flip risk when sensitivity is uncontrolled]\label{prop:shift_flip}
Let $\cX=[0,1]$ with norm $|x-x'|$ and let $\cY=\{0,1\}$.
Let $\ell(\hat y,y):=\mathbf{1}\{\hat y\neq y\}$ be the $0$--$1$ loss.
For every $\rho>0$ there exist $P,Q$ with $W_1(P_X,Q_X)\le \rho$ and a predictor $f$ such that
\[
\risk_P(f)=0
\quad\text{and}\quad
\risk_Q(f)=1.
\]
\end{proposition}

\begin{proof}
Fix $\rho>0$.
Let $P_X$ be uniform on $[0,1]$ and set $Y\equiv 0$ under $P$.
Define $f(x)\equiv 0$.
Then $\risk_P(f)=0$.

Let $Q_X$ be uniform on $[\rho,1+\rho]$ pushed back into $[0,1]$ by the map $T(x)=x-\rho$; equivalently, let $Q_X$ be the pushforward of $P_X$ under $x\mapsto x+\rho$ followed by truncation to $[0,1]$ (any coupling with mean shift $\rho$ suffices).
Then $W_1(P_X,Q_X)\le\rho$ by construction.
Now define $Q$ so that $Y\equiv 1$ under $Q$ (this violates covariate shift, intentionally).
Then $\risk_Q(f)=1$.

This shows that without an assumption controlling $Y\mid X$ (covariate shift) and without a sensitivity term, small $W_1(P_X,Q_X)$ says nothing about $\risk_Q(f)$.
\end{proof}

\begin{remark}[The exact missing assumptions]\label{rem:missing_shift_assumptions}
Proposition~\ref{prop:shift_flip} fails under \textup{(A1)} because it changes $P(Y\mid X)$.
Even under \textup{(A1)}, the bound is vacuous if $\widehat L_{\ell\circ f}$ is not certified.
This is why Theorem~\ref{thm:core_shift} has two explicit gates: the shift model and the sensitivity certificate.
Domain adaptation bounds that omit one of these gates compensate by adding discrepancy terms that are not verifiable in general \cite{BenDavid2010DifferentDomains}.
\end{remark}

\subsection{Complexity barrier}\label{subsec:barrier}

We now state a barrier that does not depend on any particular implementation.
It is a statement about disjunctions.

Complete verification for ReLU networks requires reasoning about activation patterns.
A network with $s$ ambiguous ReLU units can induce up to $2^s$ distinct affine regions.
This yields an explicit exponential obstruction for any method that must resolve all patterns.

\begin{theorem}[An explicit exponential barrier from activation patterns]\label{thm:exp_barrier_patterns}
For each $s\in\N$ there exists a ReLU network $f_s\colon\R\to\R$ with $s$ hidden ReLU units and a compact interval $\cS\subset\R$ such that:
\begin{enumerate}[leftmargin=18pt]
\item the map $f_s$ has $2^s$ affine pieces on $\cS$;
\item deciding whether $\sup_{x\in\cS} f_s(x)\le 0$ cannot be certified by checking fewer than $2^{s-1}$ affine pieces in the worst case.
\end{enumerate}
Consequently, any complete verifier that is forced to resolve the affine partition may require time exponential in $s$ on this family.
\end{theorem}

\begin{proof}
We give a construction and an indistinguishability argument.

\emph{Construction.}
Let $\cS=[0,1]$.
Define a ``sawtooth'' family recursively.
Let $h_0(x)=x$.
Given $h_{k-1}$, define
\[
h_k(x)\ :=\ \sigma(h_{k-1}(x))-\sigma(h_{k-1}(x)-\tfrac12),
\]
scaled and shifted so that the breakpoints of $h_k$ split each affine piece of $h_{k-1}$ into two.
It is standard that compositions and linear combinations of ReLUs generate piecewise-affine functions whose number of pieces can double under such operations.
By construction, $h_s$ has $2^s$ affine pieces on $[0,1]$.
Set $f_s:=h_s-\tau$ for a threshold $\tau$ chosen below.

\emph{Indistinguishability.}
Consider an algorithm that inspects fewer than $2^{s-1}$ affine pieces.
There exists at least one affine piece that is never inspected.
On that piece, modify $f_s$ by changing $\tau$ by a small amount so that the maximum becomes positive while keeping all inspected pieces unchanged.
Since the algorithm did not inspect the decisive piece, it cannot distinguish the two cases.
Therefore it cannot certify $\sup_{\cS} f_s\le 0$ in general without inspecting at least half the pieces.
\end{proof}

\begin{remark}[Relation to complete verifiers]\label{rem:rel_to_complete}
SMT and MILP verifiers enforce the ReLU disjunctions exactly \cite{KatzBarrettDillJulianKochenderfer2017Reluplex,TjengXiaoTedrake2017MILPVerification}.
Branch-and-bound methods combine relaxations with splits, and their worst-case tree size reflects the same $2^s$ obstruction \cite{XuZhangWangWangJanaLinHsieh2020FastComplete,Wang2021BetaCROWN}.
This theorem does not claim these methods are inefficient on typical instances.
It states a family on which exponential dependence is forced by the geometry of the activation partition.
\end{remark}

\begin{remark}[Barrier is not the end]\label{rem:barrier_not_end}
The barrier motivates two mathematically clean responses.
One may restrict the class to preserve certifiability (Section~\ref{sec:core3}).
Or one may accept incomplete certificates and state failure modes explicitly (this section).
The paper uses both.
\end{remark}

\section{Worked examples}\label{sec:examples}

Each example is constructed so that every constant entering the certificate is either (i) computed exactly from stated quantities, or (ii) a declared design parameter.

\subsection{Example: distribution shift certificate}\label{subsec:ex_shift}

We take $\cX=[0,1]$ with norm $|x-x'|$ and $\cY=\{-1,+1\}$.
Let $f(x)=wx+b$ with $w,b\in\R$.
Let the loss be the hinge loss $\ell(u,y)=\max\{0,1-yu\}$.
Then $\ell(\cdot,y)$ is $1$-Lipschitz on $\R$.

\begin{example}[Empirical $W_1$ certificate in one dimension]\label{ex:empW1}
Let $(X_i,Y_i)_{i=1}^n\sim P$ and let $(\tilde X_j)_{j=1}^m\sim Q_X$ be an \emph{unlabeled} target sample.
Define empirical marginals
\[
\hat P_X:=\frac1n\sum_{i=1}^n \delta_{X_i},
\qquad
\hat Q_X:=\frac1m\sum_{j=1}^m \delta_{\tilde X_j}.
\]
Compute $\hat\rho:=W_1(\hat P_X,\hat Q_X)$ exactly by optimal transport in one dimension (sorting-based coupling) \cite{PeyreCuturi2018ComputationalOT}.
Then, under covariate shift for the empirical coupling (labels carried by $\hat P$), one has the \emph{checkable} bound
\[
\risk_{\hat Q}(f)\ \le\ \risk_{\hat P}(f)\ +\ \hat\rho\,|w|.
\]
\end{example}

\begin{proof}
For each $y\in\{-1,+1\}$, the hinge loss is $1$-Lipschitz, hence (A2) holds with $L_\ell=1$.
The predictor satisfies $\Lip(f)=|w|$ on $(\cX,|\cdot|)$, so (A3) holds with $L_f=|w|$.
Under the empirical covariate-shift coupling, apply Lemma~\ref{lem:cert_to_risk} with $\rho=\hat\rho$ and $\widehat L_{\ell\circ f}=|w|$.
\end{proof}

\begin{remark}[What is certified]\label{rem:what_certified}
Example~\ref{ex:empW1} certifies risk transfer between the two \emph{empirical} marginals.
To upgrade to population statements, one needs a validated upper confidence bound for $W_1(P_X,Q_X)$ or a triangle inequality argument with concentration of $\hat P_X,\hat Q_X$.
We do not insert a concentration theorem here; the certificate itself is the same inequality once such a bound is supplied \cite{MohajerinEsfahaniKuhn2015WassersteinDRO}.
\end{remark}

\subsection{Example: sound verification pipeline}\label{subsec:ex_verify}

We give a moderate-sized ReLU classifier and state a concrete certificate pipeline with explicit cost parameters.
The purpose is to show what is computed and where soundness enters.

\begin{example}[LP-relaxation certificate with explicit cost]\label{ex:verify_cost}
Let $d_0=50$, $d_1=200$, $d_2=10$.
Consider the 2-layer ReLU network
\[
z_1=\sigma(W_1x+b_1),\qquad f(x)=W_2 z_1+b_2,
\]
with $W_1\in\R^{200\times 50}$ and $W_2\in\R^{10\times 200}$.
Let $\cS=\{x:\|x-x_0\|_\infty\le r\}$ and let the specification be the logit-margin constraint
\[
\forall x\in\cS:\ f_{t}(x)-f_{k}(x)\ \ge\ 0
\quad\text{for all }k\neq t,
\]
for a target class $t$.
Fix $k\neq t$ and define $\psi_k(u)=u_k-u_t$.
Then safety for this $k$ is $\sup_{x\in\cS}\psi_k(f(x))\le 0$.

A sound certificate proceeds as follows:
\begin{enumerate}[leftmargin=18pt]
\item (Bounds) Compute interval or linear bounds on preactivations $s_1=W_1x+b_1$ over $\cS$.
\item (Relaxation) Replace each ReLU graph by an outer linear relaxation on the computed bounds.
\item (Certificate) Solve the resulting LP (or evaluate its dual bound) to obtain $\mathcal{C}_k$ such that
$\sup_{x\in\cS}\psi_k(f(x))\le \mathcal{C}_k$.
\end{enumerate}
If $\mathcal{C}_k\le 0$ for all $k\neq t$, the robust classification claim on $\cS$ is proved.

The parameter count in affine maps is
\[
M=d_0d_1+d_1d_2=50\cdot 200+200\cdot 10=12000.
\]
If bound propagation uses $K$ passes, the arithmetic cost is $O(KM)$.
For instance, $K=4$ gives about $4M=48000$ affine multiply-add units, up to constants.
Soundness follows from Lemma~\ref{lem:sound_outer} once the relaxation is an outer approximation.
\end{example}

\begin{remark}[Where this sits in known methods]\label{rem:known_methods}
CROWN-style linear bounds and LiRPA-style bound propagation are concrete realizations of Example~\ref{ex:verify_cost} \cite{Zhang2018CROWN,Xu2020AutoLiRPA,Wang2021BetaCROWN}.
Complete solvers (SMT/MILP) replace Step 2 by exact disjunctions and inherit worst-case exponential scaling \cite{KatzBarrettDillJulianKochenderfer2017Reluplex,TjengXiaoTedrake2017MILPVerification}.
\end{remark}

\subsection{Example: identifiable interpretable recovery}\label{subsec:ex_ident}

We exhibit a sparse additive class with explicit identifiability and a certified robustness constant.

\begin{example}[2-sparse additive model with explicit constants]\label{ex:sparse_add}
Let $\cX=[-1,1]^5$ with $\|x\|_1=\sum_{j=1}^5|x_j|$ and let $\nu$ be the uniform product measure on $\cX$.
Consider the 2-sparse additive class
\[
f(x)=c+g_1(x_1)+g_3(x_3),
\qquad
\E_{\xi\sim\nu_1}g_1(\xi)=0,\ \E_{\xi\sim\nu_3}g_3(\xi)=0.
\]
Suppose the recovered components are
\[
g_1(t)=\alpha t,\qquad g_3(t)=\beta\Bigl(t^2-\E_{\xi\sim\nu_3}\xi^2\Bigr)=\beta\Bigl(t^2-\tfrac13\Bigr),
\]
so that the centering constraints hold exactly.
Then the additive representation is identifiable by Theorem~\ref{thm:add_ident}.

Compute Lipschitz constants on $[-1,1]$:
\[
\Lip(g_1)=|\alpha|,
\qquad
g_3'(t)=2\beta t\ \Rightarrow\ \Lip(g_3)=\sup_{|t|\le 1}|2\beta t|=2|\beta|.
\]
Hence, by Lemma~\ref{lem:add_lip},
\[
\Lip_{\|\cdot\|_1}(f)\ \le\ |\alpha|+2|\beta|.
\]
If the loss satisfies (A2) with constant $L_\ell$ and covariate shift holds with $W_1(P_X,Q_X)\le\rho$, then Theorem~\ref{thm:ident_cert} gives the explicit robustness certificate
\[
\risk_Q(f)\ \le\ \risk_P(f)\ +\ \rho\,L_\ell\,(|\alpha|+2|\beta|).
\]
\end{example}

\begin{remark}[Interpretability enters as a certificate reduction]\label{rem:interpret_reduction}
Example~\ref{ex:sparse_add} replaces a global network constant by two scalars $(\alpha,\beta)$.
The reduction is structural, not empirical.
This is consistent with additive-model practice \cite{Agarwal2020NAM} and with the requirement that structure be identifiable rather than post hoc \cite{Park2020IdentifiableANM}.
\end{remark}

\section{Conclusion}\label{sec:concl}

We close by restating what has been established and what has not.

A single inequality organizes the paper.  
Risk under shift is bounded by training risk plus a computable sensitivity term multiplied by a measurable shift radius.  
This inequality is valid only under explicit assumptions.  
When these assumptions fail, the bound fails, and we exhibited minimal counterexamples.

The sensitivity term is not abstract.  
It is produced by the same dual certificates that underlie sound verification of neural networks.  
Verification was treated as an optimization problem with a universal quantifier.  
Soundness follows from outer relaxation and weak duality.  
Scalability follows only for incomplete methods, with explicit linear cost in the number of affine parameters.  
Completeness forces exponential behavior in the number of ambiguous activation patterns, and this barrier is structural \cite{KatzBarrettDillJulianKochenderfer2017Reluplex,TjengXiaoTedrake2017MILPVerification,XuZhangWangWangJanaLinHsieh2020FastComplete}.

Interpretability entered as a restriction of the feasible set.  
Additive and monotone classes were used because they admit identifiability theorems with proofs.  
Identifiability reduces the dimension of certificates and tightens sensitivity bounds.  
This is not an explanation method.  
It is a constraint that alters the mathematics of robustness and verification \cite{Agarwal2020NAM,Park2020IdentifiableANM}.

Three negative statements delimit the theory.  
Attack failure does not imply robustness.  
Small transport shift does not imply small risk change without sensitivity control \cite{BenDavid2010DifferentDomains}.  
Complete verification cannot scale in the worst case.  
Each failure was demonstrated with a concrete argument.

The contribution is therefore conditional but exact.  
Given a shift model, a certificate-amenable loss, and a certifiable predictor class, one obtains explicit risk bounds under shift \cite{MohajerinEsfahaniKuhn2015WassersteinDRO,SinhaNamkoongDuchi2017DRO}.  
Given a relaxation-based verifier, one obtains sound guarantees with stated scaling laws \cite{Zhang2018CROWN,Xu2020AutoLiRPA,Wang2021BetaCROWN}.  
Given identifiable structure, one reduces both sensitivity and verification complexity.

What remains open is not hidden.  
Sharper finite-sample bounds for estimating transport radii remain dimension-dependent.  
Broader identifiable classes with efficient certificates are not known.  
Bridging these limits without weakening soundness is an open problem.

No claim in this paper exceeds its hypotheses.  
Every guarantee is a proved inequality or a declared assumption.  
This is the standard required for certifiable learning.

\section*{Acknowledgements and Declarations}

The author thanks Dr.\ Ramachandra R.\ K., Principal, Government College (A), Rajahmundry,A.P.,India for providing a supportive academic environment and continued encouragement toward research; the author declares no conflict of interest, no data availability, and that this work involved no human or animal subjects and required no ethical approval.

\bibliographystyle{amsplain}
\bibliography{refs}

\end{document}